%% file: affine_tracing.tex
\documentclass{article}

\usepackage[main, final]{neurips_2026}

\usepackage[utf8]{inputenc} 
\usepackage[T1]{fontenc}    
\usepackage{hyperref}       
\usepackage{url}            
\usepackage{booktabs}       
\usepackage{amsfonts}       
\usepackage{nicefrac}       
\usepackage{microtype}      
\usepackage{xcolor}         
\usepackage{wrapfig}

\usepackage{amsmath,amssymb,amsthm}
\usepackage{xcolor}
\usepackage{hyperref}
\usepackage[noabbrev, capitalise]{cleveref}
\usepackage{algorithm}
\usepackage{algorithmicx,algpseudocodex}
\usepackage{acro}
\usepackage{enumitem}
\usepackage[
  textsize=tiny,
  figwidth=0.99\linewidth,
]{todonotes}
\usepackage{subcaption}
\usepackage{listings}

\usepackage{xcolor}
\usepackage{xfrac}

\definecolor{codebg}{HTML}{F7F7F8}
\definecolor{codegreen}{HTML}{16A34A}
\definecolor{codepurple}{HTML}{7C3AED}
\definecolor{codeorange}{HTML}{EA580C}
\definecolor{codeblue}{HTML}{2563EB}
\definecolor{codegray}{HTML}{6B7280}

\lstdefinestyle{nicestyle}{
  backgroundcolor=\color{codebg},
  commentstyle=\color{codegreen}\itshape,
  keywordstyle=\color{codepurple}\bfseries,
  stringstyle=\color{codeorange},
  emphstyle=\color{codeblue},
  numberstyle=\tiny\color{codegray},
  basicstyle=\ttfamily\small,
  breaklines=true,
  numbers=left,
  numbersep=8pt,
  frame=single,
  rulecolor=\color{codegray!30},
  xleftmargin=1.5em,
  framexleftmargin=1.5em,
  tabsize=4,
}

\usepackage{tikz}
\usetikzlibrary{arrows.meta}

\input{macros.tex}

\input{acronyms}

\title{Affine Tracing: A New Paradigm for Probabilistic Linear Solvers}

\author{%
  Disha Hegde \\
  University of Southampton \\
  \texttt{d.hegde@soton.ac.uk} \\
  \AND
  Marvin Pförtner\\
  Tübingen AI Center, University of Tübingen \\
  \texttt{marvin.pfoertner@uni-tuebingen.de} \\
  \And
  Jon Cockayne \\
  University of Southampton \\
  \texttt{jon.cockayne@soton.ac.uk} \\
}

\begin{document}

\maketitle

\begin{abstract}
  \input{00_abstract}
\end{abstract}

\input{01_introduction}

\input{02_background}

\input{03_theory}

\input{04_affine_tracing}

\input{05_experiments}

\input{06_conclusions}

\section*{Acknowledgements}
MP gratefully acknowledges financial support by the European Research Council through ERC StG Action 757275 / PANAMA and ERC CoG Action 101123955 / ANUBIS; the DFG Cluster of Excellence ``Machine Learning - New Perspectives for Science'', EXC 2064/1, project number 390727645; the German Federal Ministry of Education and Research (BMBF) through the Tübingen AI Center (FKZ: 01IS18039A); the DFG SPP 2298 (Project HE 7114/5-1); and the Carl Zeiss Foundation (project ``Certification and Foundations of Safe Machine Learning Systems in Healthcare''); as well as funds from the Ministry of Science, Research and Arts of the State of Baden-Württemberg. The authors thank the International Max Planck Research School for Intelligent Systems (IMPRS-IS) for supporting MP. JC was supported by EPSRC grant EP/Y03533X/1.

\bibliographystyle{plainnat}
\bibliography{probabilistic_multigrid}

\newpage
\appendix
\crefalias{section}{appendix}

\input{A01_proofs.tex}
\input{A03_two_grid_mg.tex}

\input{A02_cagps.tex}

\end{document}

%% file: macros.tex
\theoremstyle{plain}
\newtheorem{theorem}{Theorem}[section]

\newtheorem{lemma}[theorem]{Lemma}

\theoremstyle{definition}
\newtheorem{definition}[theorem]{Definition}

\theoremstyle{remark}

\newcommand{\reals}{\mathbb{R}}

\newcommand{\inner}[1]{\langle #1 \rangle}

\newcommand{\defeq}{:=}
\newcommand{\normal}{\mathcal{N}}

\newcommand{\GProd}[2]{G_{#1:#2}}


%% file: acronyms.tex
\DeclareAcronym{pls}{short=PLS, long=probabilistic linear solver, long-plural=s}
\DeclareAcronym{cagp}{short=CAGP, long=computation-aware Gaussian process,long-plural=es}
\DeclareAcronym{GP}{short = GP, long = Gaussian process, long-plural=es}

%% file: 00_abstract.tex
Probabilistic linear solvers (PLSs) return probability distributions that quantify uncertainty due to limited computation in the solution of linear systems.
The literature has traditionally distinguished between Bayesian PLSs, which condition a prior on information obtained from projections of the linear system, and probabilistic iterative methods (PIMs), which lift classical iterative solvers to probability space.
In this work we show this dichotomy to be false: Bayesian PLSs are a special case of non-stationary affine PIMs.
In addition, we prove that any realistic affine PIM is calibrated.
These results motivate a focus on (non-stationary) affine PIMs, but their practical adoption has been limited by the significant manual effort required to implement them.
To address this, we introduce \emph{affine tracing}, an algorithmic framework that automatically constructs a PIM from a standard implementation of an affine iterative method by passing symbolic tracers through the computation to build an affine computational graph.
We show how this graph can be transformed to compute posterior covariances, and how equality saturation can be used to perform algebraic simplifications required for computation under specific prior choices.
We demonstrate the framework by automatically generating a probabilistic multigrid solver and evaluate its performance in the context of Gaussian process approximation.

%% file: 01_introduction.tex
\section{Introduction}

This paper focuses on solving linear systems of the form
\begin{align} \label{eq:the_system}
    A x^\star = b
\end{align}
where $A \in \reals^{d \times d}$, $b \in \reals^{d}$ and $x^\star \in \reals^{d}$.
These linear systems appear frequently in many computational routines like \ac{GP} regression \citep{rasmussen_gaussian_2006}, finite element analysis \citep{Elman2014} and Bayesian inverse problems \citep{Stuart2010}.

\Acp{pls} \citep{cockayne_probabilistic_2021,cockayne_bayesian_2019} are probabilistic numerical methods \citep{hennig_probabilistic_2022} that return a probability distribution intended to quantify the uncertainty incurred due limited computational resources spent.
This uncertainty can then be propagated through computational pipelines, e.g.\ to accelerate \acp{GP} \citep{wenger_posterior_2022,hegde_calibrated_2025}, (Bayesian) generalised linear models \citep{tatzel_accelerating_2023}, and inverse problems \citep{cockayne_bayesian_2019,vyas_randomised_2025}.

Previous works have dichotomised \acp{pls} as Bayesian and non-Bayesian.
Bayesian \acp{pls} accept a prior distribution $\mu_0$ on the solution of the linear system and condition it on information of the form $z_i = S_i A x^\star =  S_i^\top b$ where $S_i \in \reals^{d \times i}$ is a matrix of search directions with linearly independent columns.
The prior is generally chosen to be Gaussian, resulting in a Gaussian posterior distribution.
Non-Bayesian approaches, referred to as \emph{probabilistic iterative methods} (PIMs) \citep{cockayne_probabilistic_2021}, instead construct a probabilistic solver by lifting a classical iterative method to probability space.
For such an iterative method, let $P^i$ denote the map such that $x_i = P^i(x_0)$; we then define our probabilistic solver as the pushforward $\mu_i = P^i_\sharp \mu_0$ for some initial distribution $\mu_0$.
Again, typically $\mu_0$ is taken to be Gaussian; then, provided $P^i$ is an affine map, the pushforward is also Gaussian.

In this work we show this to be a false dichotomy: Bayesian methods are PIMs, and PIMs are Bayesian when the sequence $P^i$ are projection methods.
We also show that focussing on probabilistic iterative methods allows us to perform the lifting to probability space \emph{algorithmically}, rather than using pen-and-paper calculations, by transforming computational graphs constructed from an implementation of the classical underlying iterative method, in an approach we refer to as \emph{affine tracing}.
This relieves a significant bottleneck in probabilistic linear algebra: the requirement to reimplement probabilistic algorithms requires completely new, often bespoke numerical analysis and algebraic manipulations to be performed for efficiency or to perform cancellations to use particular prior structures.
In this paper we show that this process can be automated, again algorithmically, to yield efficient implementations of sophisticated modern linear solvers.

\subsection{Contributions} \label{sec:contributions}

In this paper, we make the following contributions.

\begin{itemize}
    \item We prove that Bayesian \acp{pls} are a special case of PIMs, and prove that any realistic affine PIM is calibrated.
    \item We introduce a new way to algorithmically construct PIMs, using symbolic tracing, including a mechanism for automating cancellations required for specialised prior choices.
    \item We showcase the methodology by automatically generating an implementation of a probabilistic multigrid method, whose construction would otherwise require significant mathematical derivation effort.
    We demonstrate the efficacy of this solver compared to existing, state-of-the-art \acp{pls}, by applying it in Gaussian process approximation.
\end{itemize}

\subsection{Structure of the paper}

The rest of the paper is structured as follows. In \cref{sec: background}, we recall necessary background on Bayesian PLSs and probabilistic iterative methods. \cref{sec:theory} introduces our main theoretical results showing that Bayesian PLSs are PIMs, and that affine PIMs are calibrated.
In \cref{sec:implementation} we describe our algorithm for constructing PLSs from implementations of underlying iterative methods.
We showcase a prototype implementation in \cref{sec: experiments}, and conclude with discussion in \cref{sec:discussion}.

%% file: 02_background.tex
\section{Background} \label{sec: background}

In this section we briefly discuss the two leading approaches to probabilistic linear solvers: Bayesian methods and probabilistic iterative methods.

\subsection{Bayesian Probabilistic Linear Solvers} \label{sec:bayesian_pls}

Early works on Bayesian probabilistic linear algebra focussed on what is now called the \emph{matrix-based} approach: making Bayesian inferences about matrices from their action on probe vectors \citep{Hennig2015}.
\cite{cockayne_bayesian_2019} introduced the \emph{solution-based} view, which limits attention to probabilistic linear solvers, placing a prior distribution on the unknown solution to a linear system, and conditioning on projection against search directions.
We begin by briefly outlining this setup.

For ease of computation, Bayesian \acp{pls} begin with a Gaussian prior distribution $\mu_0 \sim \mathcal{N}(x_0, \Sigma_0)$.
The $i$\textsuperscript{th} iterate $\mu_i$ is obtained by sequentially conditioning this distribution on information of the form
$s_i^\top A x^\star = s_i^\top b =: y_i$
where $(s_1, \dots, s_d)$ are a set of linearly independent search directions.
The Gaussian conditioning formula yields $\mu_i = \mathcal{N}(x_i, \Sigma_i)$ with
\begin{align*}
    x_i &= x_0 + \Sigma_0 A^\top S_i (S_i^\top A \Sigma_0 A^\top S_i)^{-1} S_i^\top (b - A x_0) \\
    \Sigma_i &= \Sigma_0 - \Sigma_0 A^\top S_i (S_i^\top A \Sigma_0 A^\top S_i)^{-1} S_i^\top A \Sigma_0
\end{align*}
where $S_i = [s_1 \cdots s_i]$.
The benefit is clear: since the required matrix inversion is $\mathcal{O}(i^3)$, the overall cost becomes quadratic in $d$ (assuming dense $A$) and cubic in the number of search directions.
This cost can be reduced by careful choice of search directions to ensure diagonality of $S_i^\top A \Sigma_0 A^\top S_i$ \citep{cockayne_bayesian_2019}, and extensions have connected the matrix- and solution-based perspectives \citep{bartels_probabilistic_2019}.

The most popular solver in this area is arguably BayesCG \citep{cockayne_bayesian_2019}, which selects $S_i$ according to the Lanczos algorithm.
Under the \emph{inverse prior} $\Sigma_0 = A^{-1}$, the iterate from the conjugate gradient method is recovered as the posterior mean, at the cost of uncertainty quantification that has repeatedly been observed to be incorrect, due to nonlinearities arising from implicit dependence of the Lanczos directions on $x^\star$ that are ignored to construct a Gaussian posterior (e.g.~\cite{hegde_calibrated_2025}).
This makes it clear that the \emph{Bayes} in BayesCG in a misnomer: the algorithm is not truly Bayesian.
There have since been various attempts to correct for this discrepancy following empirical Bayesian procedures \citep{reid_statistical_2023}, randomisation \citep{vyas_randomised_2025} and conformal approaches \citep{Hou2025}.
This highlights a central challenge with Bayesian methods: the choice of search directions $S_i$ is crucial, and approaches which lead to fast convergence tend to be poorly calibrated due to incorrect application of Bayes rule.
In the next section we discuss probabilistic iterative methods, which in some cases mitigate this issue.

\subsection{Probabilistic Iterative Methods} \label{sec:affine_pls}

\cite{cockayne_probabilistic_2021} introduced a complementary approach to that discussed in \cref{sec:bayesian_pls} which is based on \emph{lifting} an iterative linear solver into probability space.
Given an initial iterate $x_0$, an iterative linear solver produces iterates according to
\[
    x_i = P_i(x_{i-1}, x_{i-2}, \dots, x_0)
\]
for some sequence of maps $P_i : (\reals^d)^i \to \reals^d$.
We say that an iterative method is \emph{affine} if $P_i$ is an affine map, \emph{j\textsuperscript{th}-order} \citep[Chapter 3]{YoungIterativeSolutions} if $P_i$ is independent of $x_{i-j-1}, \dots, x_0$ for all $i > j$, and \emph{stationary} if $P_i = P$ for all $i$, that is, if the maps are independent of $i$.

Probabilistic iterative methods are based on the observation that, provided the maps $P_i$ are affine, one can push a joint Gaussian distribution on $(x_{i-1}, \dots, x_0)$ through $P_i$ to produce a new joint Gaussian distribution on $(x_i, \dots, x_0)$.
Any affine iterative method for solving \cref{eq:the_system} can therefore be lifted into a probabilistic iterative method for solving \cref{eq:the_system} by specifying a Gaussian prior $\mu_0 = \normal(x_0, \Sigma_0)$, from which posterior distributions over $x_1, x_2, \dots$ can be computed in terms of the matrices and vectors that define each $P_i$.

We restrict attention to first degree affine methods, i.e, $P_i(x) = G_i x + f_i$ for some matrix $G_i \in \reals^{d \times d}$ and vector $f_i \in \reals^d$.
Higher degree methods can be transformed into first degree methods on an augmented space, as shown in \citet{cockayne_probabilistic_2021}, so this approach is still general.
Clearly if all the constituent maps are affine then the composed iteration map $P^i = P_i \circ \dots \circ P_1$, that transports $x_0$ to $x_i$, is also affine.
We say that a first degree method for solving $Ax^\star = b$ has the \emph{fixed point property} if $P_i(x^\star) = x^\star$.

To lift a first degree affine iterative method into a probabilistic iterative method we follow \citet[Proposition 4]{cockayne_probabilistic_2021}.
Starting with a initial Gaussian distribution, i.e., $\mu_0 = \normal(x_0, \Sigma_0)$, the $i$\textsuperscript{th} iterate is $\mu_i = \normal(x_i, \Sigma_i)$, with
\begin{align*}
    x_i = P^i (x_0) \qquad \Sigma_i = \GProd{i}{1} \Sigma_0 \GProd{i}{1}^\top.
\end{align*}
where $\GProd{i}{1} = \prod_{j=i}^1 G_j$ and $P^i = P_i \circ \dots \circ P_1$.
Thus $x_i$ coincides with the $i$\textsuperscript{th} iterate of the underlying iterative method when initialised at the prior mean $x_0$.
\cite{cockayne_probabilistic_2021} focussed on the case of \emph{stationary methods}, i.e.\ $P_i = P$ for all $i$.
Then we have $\GProd{i}{1} = G^i$.

Whereas the quality of uncertainty quantification provided by the methods discussed in \cref{sec:bayesian_pls} is assured by virtue of being Bayesian,
these methods are clearly not in general a conditional distribution.
Previous work has justified the uncertainty quantification provided by PIMs as being \emph{strongly calibrated}, according to the following definition.

\begin{definition}[Strong Calibration] \label{def:strong_calib}
Suppose that $\mu_0 = \mathcal{N}(x_0, \Sigma_0)$ and $\mu(x^\star) = \mathcal{N}(x(x^\star), \Sigma)$ is a belief over the solution to the linear system $A x^\star = b$, with $\Sigma$ independent of $x^\star$.
Further suppose that $\Sigma$ has rank $r \leq d$, and let $R, N$ be matrices such that $R \in \reals^{d\times r}$ has column space equal to the range of $\Sigma$ and $N \in \reals^{d \times (d-r)}$ has column space equal to its kernel.

We say that $\mu(x^\star)$ is \emph{strongly calibrated} if, when $x^\star \sim \mu_0$,
\begin{align}
    (R^\top \Sigma R)^{-\frac{1}{2}} R^\top(x(x^\star) - x^\star) &\sim \mathcal{N}(0, I_r) \label{eq:cond1}\\
    N^\top (x(x^\star) - x^\star) &= 0. \label{eq:cond2}
\end{align}
\end{definition}

The intuition behind this definition is clear: if the prior is correctly specified, then the truth should be a plausible sample from the posterior.
This has the same interpretation as subjective Bayesian calibration results but does not rely on the learning procedure $\mu(x^\star)$ being a Bayesian procedure.

As mentioned, \cite{cockayne_probabilistic_2021} originally constructed probabilistic iterative methods and demonstrated their strong calibration specifically in the stationary setting.
\cite{hegde_calibrated_2025} showed that strong calibration properties propagate to downstream tasks, represented by computation-aware Gaussian processes \citep{wenger_posterior_2022}, and also demonstrated that this can result in better computational performance than using BayesCG.
However, this relies on the use of the \emph{inverse prior} and, specifically, on the property that the posterior covariance has a particular downdate form, i.e.\ $\Sigma_i = A^{-1} - D_i$ for some \emph{downdate} $D_i$ that is independent of $A^{-1}$.
Computation under the inverse prior and, specifically, computability of the downdate, is required for applications of \acp{pls} such as \acp{cagp}, as described in \cref{sec: cagp}.
While in \cref{sec:inverse_prior} we prove that this is an emergent property of \emph{any} fixed point affine iterative method, derivation of the downdate has previously required bespoke hand-computation for each method.

\subsection{Challenges for Probabilistic Linear Solvers}

The above narrative reveals a \emph{calibration-convergence tradeoff} in the literature on probabilistic linear solvers.
On the one hand, we would like to use solvers that give fast convergence, such as BayesCG.
On the other, we would like solvers that give calibrated uncertainty, such as PIMs, but existing instances of these converge slowly.

The natural path to resolving this tradeoff is to explore more sophisticated iterative methods, such as multigrid methods \citep{trottenberg_multigrid_2001} or Chebyshev semi-iterations \citep[Chapters 10-11]{YoungIterativeSolutions}.
But this is blocked by an implementation barrier: each new implementation requires bespoke, error prone algebra, with a separate implementation needed to work with the inverse prior.
As we illustrate in \cref{sec:prob_multigrid_is_hard}, the need to track cross-covariances makes implementations prohibitively complex.

We resolve this with the contributions mentioned in \cref{sec:contributions}.
We first show that the Bayesian/non-Bayesian dichotomy is false, allowing contributions to PIMs to directly contribute to Bayesian methodology.
We advance the theory on PIMs to show that every realistic affine solver is calibrated, replacing the stationarity and diagonalisability assumptions made in \citet{cockayne_probabilistic_2021} with a benign fixed-point assumption.
And, we provide a computational framework allowing implementations of PIMs to be generated automatically from implementations of underlying iterative methods.

%% file: 03_theory.tex
\section{Theoretical Results} \label{sec:theory}

We prove two key theoretical results that justify a focus on probabilistic iterative methods.
First, we show that (affine-Gaussian) Bayesian methods are instances of PIMs.

\begin{theorem} \label{thm:bayes_is_pim}
    Any Bayesian \ac{pls} with Gaussian iterates is an affine non-stationary PIM.
    Conversely, let $(P_i)$ be an affine, non-stationary PIM with the fixed-point property.
    If $\GProd{i}{1}$ is a $\Sigma_0^{-1}$-orthogonal projection onto $\textup{ker}(S_i^\top A)$, then the PIM is a Bayesian PLS with prior $\mu_0 = \mathcal{N}(x_0, \Sigma_0)$ and search directions $S_i$.
\end{theorem}

The restriction to Gaussian iterates and affine maps is fundamental to the preceding theorem; it is unclear if the same result holds when these restrictions are relaxed.
It is particularly worth noting that BayesCG does not fall under the umbrella of the above theorem because of the dependence of the search directions on the solution vector.
The result of this is that CG has maps $P^i$ that are not affine in $x_0$ or $x^\star$.
However, as has regularly been remarked, BayesCG is not truly a Bayesian algorithm \citep{cockayne_bayesian_2019,reid_statistical_2023,hegde_calibrated_2025}, therefore it is unsurprising that it does not fit within \cref{thm:bayes_is_pim}.

Before the next result, we establish the structure of first-order affine maps with the fixed point property in the next lemma.
\begin{lemma} \label[lemma]{thm:map_structure}
    Suppose that $(P_i)$ is a sequence of first-order affine maps with the fixed point property, i.e.\ that $P_i(x^\star) = x^\star$ for all $i$.
    Then there exists a matrix $\bar{M}_i$ such that
    \begin{equation*}
        P^i(x) = (I - \bar{M}_i A) x + \bar{M}_i b.
    \end{equation*}
\end{lemma}

Our next theoretical result proves strong calibration of essentially any affine iterative method.
The only requirement is that the iterative map have the fixed point property, which is a benign requirement of an iterative method.

\begin{theorem}\label{thm:affine_calibrated}
    Let $\mu_0 = \normal(x_0, \Sigma_0)$, where $\Sigma_0$ is symmetric positive definite.
    Let $P^i$ be the map transporting $x_0$ to $x_i$ for an affine iterative method for solving $A x^\star = b$, and suppose that $P^i$ has the fixed point property.
    Then the probabilistic iterative method associated with $P^i$ is strongly calibrated in the sense of \cref{def:strong_calib}.
\end{theorem}

The assumptions underlying \cref{thm:affine_calibrated} are significantly weaker than those in \citet{cockayne_probabilistic_2021}; earlier results assumed the iterative method to be stationary, and that the matrix $G$ was diagonalisable, which is typically difficult to verify.
A notable and strong limitation that remains is the affine restriction.
Again, the Gaussian algebra that this enables is fundamental to the proof, and completely different proof techniques would be needed to study calibration of PIMs constructed from non-affine maps.

%% file: 04_affine_tracing.tex
\section{Implementation: Affine Tracing} \label{sec:implementation}

There are two central challenges associated with implementation of \acp{pls} that we aim to address in this section:

\begin{enumerate}[nosep]
    \item The matrices and vectors defining each $P_i$ are often available only implicitly.
    \item Deriving posteriors for PLSs involves computing additional cross-covariance terms.
\end{enumerate}

The affine tracing methodology that we introduce resolves these two issues by running symbolic tracers through standard implementations of the iterative methods, and using the tracer to automate lifting to probability space.
Before we introduce this, we first demonstrate the issue using a multigrid method.

\subsection{Illustration: Probabilistic Multigrid Methods} \label{sec:prob_multigrid_is_hard}

Multigrid methods are among the most efficient iterative solvers for structured linear systems, achieving optimal $O(d)$ complexity for some problem classes.
These methods operate by constructing approximations to $A^1, \dots, A^n$ to the matrix $A$ on a number of different \emph{grids}, such that on the finest grid $A^1 = A$ and on the coarsest grid $A^n$ is sufficiently small that we are willing to solve linear systems involving it.
A basic two-grid cycle applies a pre-smoother (e.g. weighted Jacobi or Gauss-Seidel), computes a coarse-grid correction ($e^1$ in \cref{alg:multigrid}) by restricting the residual, solving on the coarse grid, and interpolating the correction back, and then applies a post-smoother.
More details can be found in \citet{trottenberg_multigrid_2001}.
An algorithmic description of a two-grid correction scheme is given in \cref{alg:multigrid}.

The \textsc{PreSmooth} and \textsc{PostSmooth} routines are often taken to be one or more iterations of an affine stationary iterative method (e.g.\ Gauss-Seidel).
\textsc{Restrict} and \textsc{Interpolate} map from grid $1$ to grid $2$.
\textsc{Solve} performs a linear solve.

It is possible to select all of these intermediate routines to be affine, so that \textsc{TwoGridCycle} is an affine map.
One could therefore perform the lifting described in \cref{sec:affine_pls} to produce a probabilistic two grid cycle, but this presents some immediate difficulties.
An implementation based on \cref{alg:multigrid} requires tracking cross covariances between the output of the components, significantly increasing the computational complexity.
And, while it is clear that $\textsc{TwoGridCycle}(x) = Gx + f$ for some $G$ and $f$, $G$ is available only implicitly, so it is not immediately apparent how one can compute $G \Sigma_0 G^\top$ to avoid the cross-covariance issue.
These issues are further compounded in the case of the inverse prior, in which case explicit access to $G$, to cancel terms of the form $A A^{-1}$ and identify the downdate, would be required.

To illustrate the complexity, consider the posterior covariance after a single two-grid cycle, assuming $\textsc{PreSmooth}(x) = \textsc{PostSmooth}(x) = Gx + f$, $\textsc{Restrict}(x) = R x + f'$ and $\textsc{Interpolate}(x) = P x + f''$.
The iteration matrix of the two grid cycle is then
$$
G_{\textsc{mg}} = G(I - P (A^2)^{-1} R A)G
$$
While this expression is manageable for a single two-grid cycle, the complexity rapidly becomes unmanageable as the number of cycles and complexity of the subroutines increases.
We have provided a derivation of probabilistic two-grid multigrid in \cref{sec:two_grid_derivation} both to illustrate this, and as a benchmark for experiments in \cref{sec: experiments}.
The complexity of these derivations is precisely the implementation barrier that affine tracing resolves.

\subsection{Affine Tracing} \label{sec:affine_tracing}

The solution we propose is to build an \emph{affine tracer}.
This involves passing a bespoke symbolic tracer into a standard implementation of an iterative method in place of $x_0$.
By operator overloading, the tracer records all affine operations applied to it into a computational graph, which we can then manipulate to compute both $x_i$ and $\Sigma_i$.
To compute $x_i$, we simply pass $x_0$ through the graph, applying all the recorded affine operations, to compute $x_i$.
This could also be performed during the initial traversal to construct the computational graph.
At present the tracer \emph{assumes} that the traced program only contains affine operations, and will raise errors if a non-affine operation is encountered.
Handling non-affine operations would be a significant extension to the methodology, and is left for future work.

To compute $\Sigma_i$ we perform a straightforward transformation of the computational graph, replacing all nodes involving shift operations (i.e. $x \mapsto x + v$) with identity nodes ($x \mapsto x$), and preserving all other nodes.
This transformation results in a graph representation of the matrix $\GProd{i}{1}$, otherwise defined only implicitly.
We can then compute $\Sigma_i$ with two forward passes through the transformed graph, once to compute $\GProd{i}{1} \Sigma_0$, and then again to compute $\GProd{i}{1} (\Sigma_0 \GProd{i}{1}^\top)$, i.e.\ passing through the transpose of the first computation.
Alternatively, where we require only a low-dimensional projection of the posterior described by the matrix $V \in \reals^{d' \times d}$, we can perform a single \emph{reverse} traversal of the graph to compute $Z_i^\top = \GProd{i}{1}^\top V^\top$, and then compute $V\Sigma_i V^\top = Z_i \Sigma_0 Z_i^\top$.

An example of applying the affine tracer in the context of Gauss-Seidel is given in \cref{fig:tracer_output}.
Next we turn our attention to algebraic simplification of the constructed graph.

\begin{figure}[t]
    \centering
    \begin{subfigure}[c]{0.48\textwidth}
        \centering
\begin{lstlisting}[language=Python]
def gauss_seidel(A, b, x):    
  L = tril(A)    
  U = triu(A, k=1)    
  y = U @ x    
  z = b - y    
  return solve_triangular(L,z)
\end{lstlisting}
    \end{subfigure}
    \begin{subfigure}[c]{0.48\textwidth}
        \centering
        \input{gs_graph.tex}
    \end{subfigure}
    \caption{Left: Python Gauss-Seidel routine. Right: Output graph from the affine tracer. Green nodes denote linear operations. Orange nodes denote shifts; these are replaced with identity operators for the transformation to compute the covariance described in \cref{sec:affine_tracing}.\vspace{-10pt}}
    \label{fig:tracer_output}
\end{figure}

\subsection{Inverse Priors} \label{sec:inverse_prior}

A remaining complication is the need to perform computation under the inverse prior.
We first introduce a theoretical result that ensures for any reasonable affine iterative method, the posterior has a structure that makes computation with the inverse prior feasible, and proceed to perform analysis to identify what operations are required for computation with the inverse prior.
Then, in \cref{sec:egg} we show how we can apply this to transform the computational graph constructed in \cref{sec:affine_tracing} to apply the action of $\bar{M}_i$ algorithmically.

From \cref{thm:map_structure} it is clear, at least abstractly, that the posterior covariance has the form $\Sigma_i = (I - \bar{M}_i A) \Sigma_0 (I - \bar{M}_i A)^\top$.
Therefore, if $\Sigma_0 = A^{-1}$, we have the simplification:
\begin{align}
    \Sigma_i = A^{-1} - \bar{M}_i - \bar{M}_i^\top + \bar{M}_i A \bar{M}_i^\top = A^{-1} - D_i. \label{eq: posterior_downdate}
\end{align}
In most applications of the inverse prior, only the downdate term $D_i$ is required for computation.
The above simplification therefore shows that for a generic fixed point affine iterative method, the required downdate structure is recovered.

The challenge is that the affine tracer only provides access to the left- or right- action of $\GProd{i}{1}$.
Translating this into the action of $\bar{M}_i$ requires \emph{explicit} access to $\GProd{i}{1}$ to both recognise the structure that relates it to $\bar{M}_i$ and perform the required cancellations with $A^{-1}$.
In this section, we first identify the algebraic structure associated with $D_i$.

We first introduce $S(M) = M + M^\top - MAM^\top$ and set $D_0 = 0$.
Since $P_i x = G_i x + f_i$, we have
\begin{align*}
    \Sigma_i &= G_i \Sigma_{i-1} G_i^\top \\
    &= G_i (A^{-1} - D_{i-1}) G_i^\top \\
    &= (I - M_iA) A^{-1} (I - M_iA)^\top - G_i D_{i-1} G_i^\top \\
    &= A^{-1} - M_i - M_i^\top + M_i A M_i^\top - G_i D_{i-1} G_i^\top \\
    \implies D_i &= S(M_i) + G_i D_{i-1} G_i^\top \\
    &= S(M_i) + \sum_{j=1}^{i-1} \GProd{i}{j+1} S(M_j) \GProd{i}{j+1}^\top.
    \vspace{-5pt}
\end{align*}
To limit the computational complexity of this approach, we assume that the goal is to compute a low-dimensional summary statistic of the posterior, described by $V \in \reals^{d' \times d}$ with $d' \ll d$; therefore
\begin{equation}
    V D_i V^\top = V S(M_i) V^\top + \sum_{j=1}^{i-1} V \GProd{i}{j+1} S(M_j) \GProd{i}{j+1}^\top V^\top. \label{eq: iterative_downdate}
\end{equation}
The quantities $Z_{i,j} := V \GProd{i}{j}$ can be obtained by application of the tracer from \cref{sec:affine_tracing}.
The additional requirement for the inverse prior is therefore to evaluate terms of the form $Q_i(V) = V S(M_i) V^\top$ for each $i$.
Since $Q_i(V) = V (M_i + M_i^\top - M_i A M_i ^\top) V^\top$, it is clearly sufficient to evaluate $M_i V^\top$ and $V M_i$ for arbitrary $V$, since $Q_i(V)$ can be obtained by multiplying this quantity against $V^\top $ and $(M_i V^\top)^\top$.
That is, we can \emph{isolate attention to simplifying the computational subgraph} defined by the isolated iteration maps $P_i$, rather than needing to analyse the full graph defined by $P^i$.

\paragraph{Stationary Simplifications}
In the case of a stationary method the computations above simplify significantly, enabling more efficient computation.
We then have $\GProd{i}{j} = G^{i-j}$, so that
\begin{align*}
V D_i V^\top &= V S(M) V^\top + \sum_{j=0}^{i-1} VG^j S(M) (G^j)^\top V^\top \\
&= V D_{i-1} V^\top + V G^{i-1} S(M) (G^{i-1})^\top V^\top.
\end{align*}
By caching the intermediate quantity $Z_{i-1} = V G^{i-2}$, we can compute $Z_i = Z_{i-1} G$ using affine tracing.
Then it remains to compute $Q(Z_i)$, which is discussed in the next section.

\subsubsection{Inverse Cancellation with Equality Saturation} \label{sec:egg}

In this section we will focus on a single iteration within the full, potentially nonstationary iteration pipeline, and will therefore drop the subscript $i$ associated with the iteration map.

The above narrative makes it clear that the only remaining task is to simplify the computational subgraph associated with a single step $P(x) = G x + f$, constructed in \cref{sec:affine_tracing}, to provide access to the matrix multiplications $V M$ and $M V^\top$ for arbitrary $V \in \reals^{d' \times d}$.
We accomplish this using equality saturation \citep{Tate2011}, with a custom ruleset defined in the \texttt{python} library \texttt{egglog} \citep{Shanabrook2023EgglogPython}, providing bindings to the \texttt{Rust} library \texttt{egg} \citep{Willsey2021}.

In brief, equality saturation attempts to find an optimal equivalent representation of a directed acyclic graph, such as a computational graph.
To use equality saturation, one first defines a set of subgraph structures that are considered to be equivalent.
For this paper, this amounts to providing a description of the algebraic structure associated with matrix arithmetic, that is, properties of matrix multiplication (e.g.\ associativity, distributivity), addition (e.g.\ as for multiplication, as well as commutativity), scalar multiplication, and relationships of these operations with identity and zero elements (i.e. $A^{-1} A = I$, $A - A = 0$).
With these equivalences defined, rewrite rules are applied iteratively, populating an \emph{e-graph} with equivalent representations.
We associate an approximate cost with each operation, based in this paper on matrix sizes and computational complexities rather than precise flops.
An extraction pass then selects the minimum-cost representative.
This process avoids the pathologies associated with a greedy rewrite by maintaining all equivalent graphs until extraction.

Concretely, we translate the computational graph constructed by the affine tracer into an \texttt{egglog} representation of the computation $V G$ by traversing the graph while ignoring shift operations.
We then augment the \texttt{egglog} representation by adding a symbolic right multiplication by $A^{-1}$ and subtraction of $V A^{-1}$, so that it represents the computation $V G A^{-1} - V A^{-1}$.
To emphasise, this does not involve computation of $A^{-1}$, only construction of an e-graph node representing the computation.
We then assign dimension-specific costs to the constituent matrices underpinning $G$, $V$ and $A^{-1}$, and apply equality saturation.
Since the cost associated with $A^{-1}$ is cubic in the dimension $d$, provided $d'$ is small enough this yields a numerically optimal implementation of the required computation $VM$ (under the equivalence classes supplied) due to the cancellation of $A^{-1}$ implied by \cref{thm:map_structure}.

%% file: gs_graph.tex

\begin{tikzpicture}[
    node distance=1.5cm,
    >=Stealth,
    darknode/.style={
        circle, 
        minimum size=1cm, 
        draw=none, 
        fill={rgb,255:red,64;green,94;blue,82}, 
        text=white, 
        font=\bfseries\large,
        inner sep=0pt,
    },
    shiftnode/.style={
        circle, 
        minimum size=1cm, 
        draw=none, 
        fill={rgb,255:red,227;green,168;blue,56}, 
        text=white, 
        font=\bfseries\large,
        inner sep=0pt,
    },
    arrow/.style={->, thick, color=black!70},
    looparrow/.style={->, thick, color=black!70},
]

\node[darknode] (U) {$U$};
\node[darknode, right of=U] (neg) {$-1$};
\node[shiftnode, right of=neg] (addb) {$+b$};
\node[darknode, right of=addb] (Linv) {$L^{-1}$};

\draw[arrow] (U) -- (neg);
\draw[arrow] (neg) -- (addb);
\draw[arrow] (addb) -- (Linv);


\end{tikzpicture}

%% file: 05_experiments.tex
\section{Experiments} \label{sec: experiments}
A \texttt{python} implementation of this solver is available on GitHub\footnote{\url{https://github.com/hegdedisha/affine_tracer}}. All the timings reported in this section were obtained from computations performed in University of Southampton’s Iridis X cluster on a single node with 32 CPU cores.

\subsection{Performance of Affine Tracer} \label{sec: time_affine_tracer}
We first compare the performance of our affine tracer to a hand implemented \ac{pls}. 
We consider the implementation of a 2-grid probabilistic multigrid solver based on lifting \cref{alg:multigrid}.
Hand-computation of the posterior is possible, and is provided in \cref{sec:two_grid_derivation}. 
We consider solving linear systems $A x = b$ of varying sizes $d$, where $A$ is a Matèrn $\sfrac{3}{2}$ kernel with lengthscale $0.1$ and amplitude $1$, and $b$ is vector of ones. 
We use 3 iterations of Gauss-Seidel as both pre- and post-smoother, a nearest neighbour restrictor using one neighbour and a nearest neighbour inverse-distance weighted interpolator with two neighbours.
The coarse grid operator $A^2$ is then calculated as the product of restrictor, fine grid operator and interpolator.
The fine grid is taken to a uniformly spaced one dimensional grid of size $d$ in $[0, 1]$, and the coarse grid taken to be a uniformly spaced one dimensional grid of size $d/5$ in $[0, 1]$, where $d$ is selected to ensure $d/5$ is an integer. Initial distribution is taken to be $\mathcal{N}(0, I)$ and the posterior of the PLS is computed after 20 iterations.

\cref{fig: time_affine_tracer} compares wall-time taken by the hand-implemented version (\texttt{ProbMG}) and affine traced version (\texttt{AffineTracer}), averaged over 10 runs. We see that the affine traced version is always faster than its counterpart, and time gap increases with the dimension of the system being solved. This is due to the additional cross-covariance computations required for the hand-implemented version.

\begin{figure}
\centering
\includegraphics[width=0.5\textwidth]{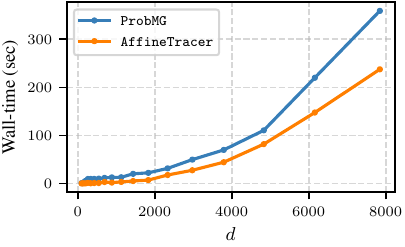}
\caption{Times for probabilistic linear solvers as described in \cref{sec: time_affine_tracer}.} \label{fig: time_affine_tracer}
\end{figure}

\subsection{Synthetic \ac{cagp} Problem} \label{sec: synthetic}

We now test the performance of the affine traced probabilistic linear solver when using the inverse prior in the \ac{cagp} setting. For background information on CAGPs, see \cref{sec: cagp}. 
We first consider a synthetic well-specified setting where the true function is generated as a sample of the prior. 
The prior mean is taken to be 0, with prior covariance a Matèrn $\sfrac{3}{2}$ kernel with lengthscale $0.8$ and amplitude $1$. Observational noise variance is $0.1$ for data generation. Training points are a uniform 2 dimensional grid in $[0, 1]$ of size $100 \times 100$. 
For multigrid, the coarsest grid is computed to be the uniform 2 dimensional grid in $[0, 1]$ of size $30 \times 30$, and an intermediate uniform grid of size $40 \times 40$ is used for 3-grid multigrid. 
Test data is computed on the $5 \times 5$ uniform grid.

Four versions of \acp{cagp} are compared as \texttt{CAGP-CG}, \texttt{CAGP-GS}, \texttt{CAGP-MG-2} and \texttt{CAGP-MG-3} that use probabilistic Gauss-Seidel, BayesCG, probabilistic 2-grid multigrid and probabilistic 3-grid multigrid as the corresponding \acp{pls} respectively. 
Probabilistic 2-grid multigrid and probabilistic 3-grid multigrid are implemented using the affine tracer, and the required downdate is computed using the \texttt{egglog} framework described in \cref{sec:egg}. The detailed algorithm is provided in \cref{alg: CAGP-SIS}. 
We also include the GP posterior as a baseline.

\begin{figure}
\includegraphics[width=\textwidth]{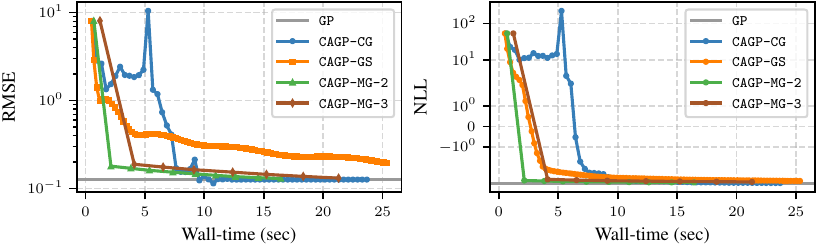}
\caption{RMSE and NLL for the synthetic problem in \cref{sec: synthetic}.} \label{fig: synthetic:rmse_nll}
\end{figure}

\cref{fig: synthetic:rmse_nll} shows the root mean square error (RMSE) and negative log likelihood (NLL) from all the methods against wall-time. 
Faster convergence of both multigrid variants is clear, and more than compensates for the higher cost of each iteration.
This suggests that multigrid iterations with affine tracing may provide state-of-the-art performance in CAGPs.

\subsection{ERA5 Regression} \label{sec: era5}

\begin{figure}[t!]
\includegraphics[width=\textwidth]{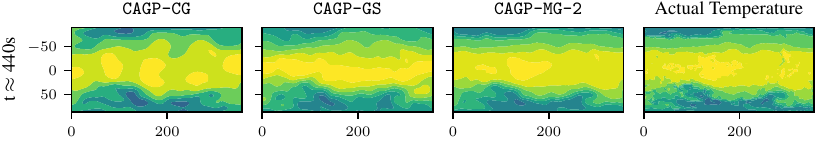}
\caption{Posterior means for the ERA5 regression problem in \cref{sec: era5}. \vspace{-10pt}} \label{fig: era5:mean}
\end{figure}

\begin{figure}[t!]
\includegraphics[width=\textwidth]{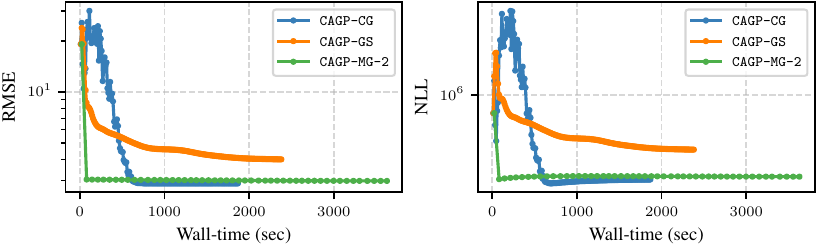}
\caption{RMSE and NLL for the ERA5 regression problem in \cref{sec: era5}.} \label{fig: era5:rmse_nll}
\end{figure}

Finally we examine the performance of \texttt{CAGP-MG-2} in a large regression problem. We consider the ERA5 2 meter temperature dataset \citep{era5_2023}. We only consider the spatial data for a single time point on 1st January 2026 at 00.00. We use an equally spaced $135 \times 270$ grid of latitudes and longitudes leading to a training set of size 36450. A coarse grid of $60 \times 30$ was used for multigrid. 
The test set is taken to be an equally spaced $45 \times 90$ grid. 

\cref{fig: era5:mean} shows the posterior mean fields after expending approximately the same amount of computational time. We see that \texttt{CAGP-MG-2} is much smoother than \texttt{CAGP-CG} and \texttt{CAGP-GS} as well as reproducing more fine structure of the actual temperature field. 
\cref{fig: era5:rmse_nll} explains this behaviour as we see the RMSE of \texttt{CAGP-MG-2} drops rapidly, while \texttt{CAGP-CG} and \texttt{CAGP-GS} show much slower convergence. 
Interestingly, \texttt{CAGP-MG-2} appears to stagnate after first few iterations. 
It is unclear whether it is continuing to converge slowly, or it has reached convergence, but notably the performance of \texttt{CAGP-CG} is similar when this algorithm terminates.

%% file: 06_conclusions.tex
\section{Conclusions} \label{sec:discussion}
In this paper we have significantly advanced understanding of PIMs by (a) showing that Bayesian \acp{pls} \emph{are} PIMs (b) proving that essentially all affine PIMs are calibrated and (c) providing an algorithm that automates implementation of affine PIMs, including computational graph simplification under the inverse prior.
This opens up a much larger design space for \acp{pls}, since practitioners can now experiment with sophisticated iterative methods without needing to produce bespoke implementations.

Several limitations of this machinery are apparent.
First, the equality saturation approach was found to be rather brittle in practice.
It would be much more appealing to construct the entire computational graph associated with all the \ac{pls} iterations and pass this to an equality saturation algorithm, than to break down to the level discussed in \cref{sec:egg}.
Unfortunately this results in graphs that are too large and complex for the simplification to be performed.
Future research could investigate bespoke simplification algorithms for the (highly structured) simplifications required for the inverse prior.
Also left unexplored is the numerical impact of equality saturation.
The optimal algorithm identified by the e-graph may not be numerically stable.

The most significant outstanding questions concern relaxing the affine assumption, which essentially excludes Krylov methods, the workhorse of solution of sparse linear systems.
Krylov methods can still be lifted into PIMs by the construction detailed in \cref{sec:affine_pls}; this was explored in \citet{cockayne_probabilistic_2021}, which used samples from $P^i_\sharp \mu_0$ to show experimental evidence for calibration.
The central challenge is tractability, since a sampling algorithm is not numerically appealing.
Future work might explore Gaussian approximations.
A secondary question concerns theory.
The theoretical results presented in this paper rely intrinsically on affine-Gaussian algebra.
To extend towards Krylov methods would require different proof techniques, perhaps in the more general theoretical framework exposed in \citet{cockayne_testing_2022}.

%% file: A01_proofs.tex
\section{Proofs}

We begin with a brief lemma before proving \cref{thm:bayes_is_pim}.

\begin{lemma} \label[lemma]{lem: projection}
    Let $M \in \reals^{d \times i}$ have full column rank.
    The orthogonal projection onto $\textup{range}(M)$ with respect to the inner product induced by a symmetric positive definite matrix $S$ is given by    
    \begin{align*}
        P_M = M \left(M^\top S M\right)^{-1} M^\top S.
    \end{align*}
    The orthogonal projection onto the $S$-orthogonal complement of $\textup{range}(M)$ is given by $I - P_M$.
\end{lemma}
\begin{proof}
    We first verify that $P_M$ and $I-P_M$ are projections.
    This is clear using \cite[Section II-7, Theorem 2]{Yosida1995} since 
    \begin{align*}
        P_M^2 &= M \left(M^\top S M\right)^{-1} M^\top S M \left(M^\top S M\right)^{-1} M^\top S \\
        &= M \left(M^\top S M\right)^{-1} M^\top S \\
        &= P_M.
    \end{align*}
    Similarly, $(I - P_M)^2 = I - 2 P_M + P_M^2 = I - P_M$.

    To show that $P_M$ is an orthogonal projection it is equivalent to show it is self-adjoint in the $S$-geometry \cite[Section II-7, Theorem 2]{Yosida1995}.
    Clearly
    \begin{align*}
        \langle  x, P_M y \rangle_S &= x^\top S P_M y \\
        &= x^\top S M \left(M^\top S M\right)^{-1} M^\top S y \\
        &= x^\top P_M^\top S y \\
        &= \langle P_M x, y \rangle_S.
    \end{align*}
    The same follows for $(I - P_M)$ since $\langle x, I - P_M y \rangle_S = \langle x, y\rangle_S - \langle x, P_M y\rangle_S$.

    Finally note that if $v \in \textup{range}(P_M)$ and $v' \in \textup{range}(I - P_M)$ then 
    \begin{align*}
        \inner{v, v'}_S &= \inner{P_M v, (I - P_M) v'}_S \\
        &= \inner{P_M v, v'}_S - \inner{P_M v, P_M v'}_S \\
        &= \inner{v, P_M v'}_S - \inner{v, P_M v'}_S = 0
    \end{align*}
    where the last line is due to self-adjointness.
    This shows that $\textup{range}(I - P_M)$ is the $S$-orthogonal complement of $\textup{range}(P_M)$.
\end{proof}

\begin{proof}[Proof of \cref{thm:bayes_is_pim}]

We first prove that any Bayesian PLS is an affine non-stationary PIM. Given as prior $\mu_0 \sim \normal(x_0, \Sigma_0)$, from \citet{cockayne_bayesian_2019} we have that the posterior $\mu_i^{\textsc{bayes}}$ of a Bayesian PLS defined by search directions $S_i$ is given by
\begin{align}
    \mu_i^{\textsc{bayes}} &\sim \normal(x_i^{\textsc{bayes}}, \Sigma_i^{\textsc{bayes}}) \\
    x_i^{\textsc{bayes}} &= x_0 + \Sigma_0 A^\top S_i (S_i^\top A \Sigma_0 A^\top S_i)^{-1} S_i^\top (b - Ax_0) \\
    \Sigma_i^{\textsc{bayes}} &= \Sigma_0 A^\top S_i (S_i^\top A \Sigma_0 A^\top S_i)^{-1} S_i^\top A \Sigma_0
\end{align}
Now consider the $i$\textsuperscript{th} iterate of a non-stationary affine PIM (starting with the same prior) given by $P^i(x) = \GProd{i}{1} x + \bar{f}_i$ with
\begin{align*}
    \GProd{i}{1} &= I - \Sigma_0 A^\top S_i (S_i^\top A \Sigma_0 A^\top S_i)^{-1} S_i^\top A\\
    \bar{f}_i &= \Sigma_0 A^\top S_i (S_i^\top A \Sigma_0 A^\top S_i)^{-1} S_i^\top b
\end{align*}
The $i$\textsuperscript{th} iterate of this PIM is given by
\begin{align}
    \mu_i &\sim \normal(x_i, \Sigma_i)  \notag\\
    x_i &=  \GProd{i}{1} x_0 + \bar{f}_i \notag\\
    &= (I - \Sigma_0 A^\top S_i (S_i^\top A \Sigma_0 A^\top S_i)^{-1} S_i^\top A) x_0 + \Sigma_0 A^\top S_i (S_i^\top A \Sigma_0 A^\top S_i)^{-1} S_i^\top b \notag\\
    &= x_0 + \Sigma_0 A^\top S_i (S_i^\top A \Sigma_0 A^\top S_i)^{-1} S_i^\top (b - Ax_0) \notag\\
    &= x_i^{\textsc{bayes}} \label{eq: Bayes_PIM_mean}\\
    \Sigma_i &= \GProd{i}{1} \Sigma_0 \GProd{i}{1}^\top \notag\\
    &= ( I - \Sigma_0 A^\top S_i (S_i^\top A \Sigma_0 A^\top S_i)^{-1} S_i^\top A) \Sigma_0 ( I - \Sigma_0 A^\top S_i (S_i^\top A \Sigma_0 A^\top S_i)^{-1} S_i^\top A)^\top \notag\\
    &= \Sigma_0 - \Sigma_0 A^\top S_i (S_i^\top A \Sigma_0 A^\top S_i)^{-1} S_i^\top A \Sigma_0 - \Sigma_0 A^\top S_i (S_i^\top A \Sigma_0 A^\top S_i)^{-1} S_i^\top A \Sigma_0 \notag\\ &\quad \quad + \Sigma_0 A^\top S_i (S_i^\top A \Sigma_0 A^\top S_i)^{-1} S_i^\top A \Sigma_0 A^\top S_i (S_i^\top A \Sigma_0 A^\top S_i)^{-1} S_i^\top A \Sigma_0 \notag\\
    &= \Sigma_0 - \Sigma_0 A^\top S_i (S_i^\top A \Sigma_0 A^\top S_i)^{-1} S_i^\top A \Sigma_0 \notag\\ 
    &= \Sigma_i^{\textsc{bayes}} \label{eq: Bayes_PIM_cov}
\end{align}
This shows that the Bayesian PLS is in fact, a non-stationary affine PIM.

For the converse consider a PIM for which $\GProd{i}{1}$ is an orthogonal projection onto $\textup{ker}(S_i^\top A)$ with respect to the $\Sigma_0^{-1}$ inner-product. As $\textup{ker}( S_i^\top A)$ is the $\Sigma_0^{-1}$-orthogonal complement of $\textup{range}(\Sigma_0 A^\top S_i)$, using \cref{lem: projection} we have
\begin{align*}
    \GProd{i}{1} &= I - \Sigma_0 A^\top S_i \left(S_i^\top A \Sigma_0 \Sigma_0^{-1} \Sigma_0 A^\top S_i\right)^{-1} S_i^\top A \Sigma_0 \Sigma_0^{-1} \\
    &= I - \Sigma_0 A^\top S_i \Lambda_i^{-1} S_i^\top A
\end{align*}
where $\Lambda_i = S_i^\top A \Sigma_0 A^\top S_i$.

As this nonstationary PIM is assumed to have the fixed point property, which is equivalent to the consistency property from \cite[Section~3.1]{YoungIterativeSolutions}, $\bar{f}_i$ is given by \citet[Theorem 9.1.2]{YoungIterativeSolutions} as 
\begin{align*}
    \bar{f}_i  &= (I - \GProd{i}{1}) A^{-1} b\\
    & = \Sigma_0 A^\top S_i \Lambda^{-1} S_i^\top b
\end{align*}

Thus, given $\GProd{i}{1} = I - \Sigma_0 A^\top S_i \Lambda^{-1} S_i^\top A$ and $\bar{f}_i =  \Sigma_0 A^\top S_i \Lambda^{-1} S_i^\top b$, as seen in \cref{eq: Bayes_PIM_mean,eq: Bayes_PIM_cov}, we have that the PIM iterates are equal to Bayesian PLS iterates defined by $S_i$.

\end{proof}

\begin{proof}[Proof of \cref{thm:map_structure}]
    Since $(P_i)$ has the fixed point property, each $P_i$ is \emph{consistent} in the notation of \citet[Section 3.1]{YoungIterativeSolutions}.
    By \citet[Theorem 3.2.6]{YoungIterativeSolutions}, since $P_i$ is affine we have that $P_i(x) = (I-M_i A) x + \tilde{f}_i$ for some $M_i$, $\tilde{f}_i$.
    The result follows by induction; assume that $P^{i-1} = (I - \bar{M}_{i-1} A) x + \bar{f}_{i-1}$.
    Then 
    \begin{align*}
        P^i(x) &= (P_i \circ P^{i-1})(x) \\
        &= P_i((I - \bar{M}_{i-1} A)x + \bar{f}_{i-1}) \\
        &= (I - M_i A)((I - \bar{M}_{i-1} A)x + \bar{f}_{i-1}) + \tilde{f}_i \\
        &= (I - [\bar{M}_{i-1} + M_i - M_i A \bar{M}_{i-1}] A)x + (I - M_i A)\bar{f}_{i-1} + \tilde{f}_i 
    \end{align*}
    as required.
    Additionally, the map $(P^i)$ has the fixed point property from \citet[Theorem 9.1.1]{YoungIterativeSolutions}.
    Thus for any $i = 1, 2, \dots$, we have that 
    \begin{align*}
        x^\star &=  G_{i:1}(x^\star) + \bar{f}_{i} \\
        \implies x^\star &= (I - \bar{M}_i A)(x^\star) + \bar{f}_{i} \\
        \implies x^\star &= x^\star - \bar{M}_i b +  \bar{f}_{i}  \\
        \implies \bar{f}_{i} &= \bar{M}_i b
    \end{align*}
    as required.
\end{proof}

\begin{proof}[Proof of \cref{thm:affine_calibrated}]

This proof closely follows the proof in \citet[Proposition 10]{cockayne_probabilistic_2021}.
Throughout we suppress dependence on the iteration index $i$ since this is not relevant to the proof, and we simply use $P^i(x) = P(x) = Gx + f$.
We retain the terminology $\mu_0 = \mathcal{N}(x_0, \Sigma_0)$, but use the notation $\mu = \mathcal{N}(x, \Sigma)$ for the posterior $P_\sharp \mu_0$.
We also use the notation $A^\frac{1}{2}$ to denote a \emph{factor} of $A$ with the property $A^\frac{\top}{2} A^\frac{1}{2} = A$, not necessarily a symmetric square root.

From \cref{thm:map_structure}, we see that $P(x) = (I - MA)x + Mb$ for some matrix $M$, and thus $G = I - MA$ is independent of $b$. Thus when $x^\star \sim \mu$, $\Sigma = G \Sigma G^\top$ is independent of $x^\star$, and depends only on $A$.

We now compute appropriate bases $R$ and $N$ for the range and null space of the posterior covariance.

\paragraph{Null Space.}
Let $\mathrm{rank}(G) = \tilde{r}$. From \citet[Theorem 5.1.1]{bernstein_matrix_2009}, there exist non-singular $Y, W \in \reals^{d\times d}$ such that
\begin{align*}
    G = Y \Omega W
\end{align*}
where $\Omega$ is of the form
\begin{align*}
    \Omega =  \begin{bmatrix}
        I_{\tilde{r}} & 0_{\tilde{r} \times (d - \tilde{r})}\\
        0_{(d-\tilde{r}) \times \tilde{r}} & 0_{(d-\tilde{r}) \times (d-\tilde{r})}
    \end{bmatrix}.
\end{align*} 
We partition $Y$ and $W$ as
\begin{align*}
    Y = \begin{bmatrix}
        Y_1 & Y_2
    \end{bmatrix} \qquad
    W = \begin{bmatrix}
        W_1^\top \\ W_2^\top
    \end{bmatrix}
\end{align*}
with $Y_1, W_1 \in \reals^{d \times \tilde{r}}$ and $Y_2, W_2 \in \reals^{d \times (d -\tilde{r})}$. Thus $G = Y_1 W_1^\top$. 
Similarly we partition $Y^{-1}$ as
\begin{align*}
    Y^{-1} = \begin{bmatrix}
         \tilde{Y}_1^\top \\ \tilde{Y}_2^\top
    \end{bmatrix}
\end{align*}
with $\tilde{Y}_1 \in \reals^{d \times \tilde{r}}$ and $\tilde{Y}_2 \in \reals^{d \times (d -\tilde{r})}$. As $Y^\top$ has linearly independent rows, $\tilde{Y}_1$ and $\tilde{Y}_2$ have full column-rank. Using $Y^{-1} Y = I_d$ we get $\tilde{Y}_2^\top Y_1 = 0$. We now have
\begin{align*}
    G^\top \tilde{Y}_2 &= \begin{bmatrix}
        W_1 & W_2
    \end{bmatrix} \begin{bmatrix}
         I_{\tilde{r}} & 0_{\tilde{r} \times (d - \tilde{r})}\\
        0_{(d-\tilde{r}) \times \tilde{r}} & 0_{(d-\tilde{r}) \times (d-\tilde{r})}
    \end{bmatrix} \begin{bmatrix}
        Y_1^\top \\ Y_2^\top
    \end{bmatrix} \tilde{Y}_2 \\
    &= W_1 Y_1^\top \tilde{Y_2} \\
    &= 0
\end{align*}
This implies range($\tilde{Y}_2$) $\in$ ker($G^\top$). As rank($G$) = $\tilde{r}$, rank($G^\top$) = $\tilde{r}$ as well. From the rank nullity theorem \citep[Corollary 2.5.5]{bernstein_matrix_2009}, we have that nullity($G^\top$) = $d - \tilde{r}$. Thus $\tilde{Y}_2$, whose rank is $d - \tilde{r}$, must span the entire ker($G^\top$). This gives us ker($G^\top$) = range($\tilde{Y}_2$). 

Defining $Q = G \Sigma_0^{1/2}$, we have $\Sigma = Q Q^\top$. 
It is straightforward to show that $\mathrm{ker}(\Sigma) = \mathrm{ker}(Q)$:
\begin{align*}
    v &\in \text{ker}(\Sigma) \\
    \iff \Sigma v &= 0 \\
    \iff v^\top \Sigma v &= 0 \qquad \qquad \qquad (\text{Reverse implication true as } \Sigma \text{ is positive semi-definite}) \\
    \iff (Q^\top v)^\top Q^\top v &= 0 \\
    \iff Q^\top v &= 0 \\
    \iff v &\in \text{ker}(Q^\top).
\end{align*}

Thus 
\begin{align*}
    \text{ker}(\Sigma) = \text{ker}(Q^\top) = \text{ker}(\Sigma_0^{1/2}G^\top) = \text{ker}(G^\top) = \text{range}(\tilde{Y}_2)
\end{align*}
The third equality holds from \citep[Proposition 2.6.3]{bernstein_matrix_2009} as $\Sigma_0^{1/2}$ is invertible. 

\paragraph{Range.} 
Since $Y_1$ has linearly independent columns (as it is based on a partitioning of the full-rank $Y$), it has full column rank, and thus is left-invertible \citep[Theorem 2.6.2]{bernstein_matrix_2009}. 
Similarly, $W_1^\top$ is right-invertible, and thus $W_1$ is left-invertible. From \citet[Proposition 2.6.3]{bernstein_matrix_2009},
\begin{align*}
    \text{range}(G) &= \text{range}(Y \Omega W) \\
    &= \text{range}(Y_1 W_1^\top) \\
    &= \text{range}(Y_1).
\end{align*}

We also have that range($\Sigma^\top$) is the orthogonal complement of ker($\Sigma$), which is equal to ker($G^\top$). And ker($G^\top$) is the orthogonal complement of range($G$). This gives us
\begin{align*}
    \text{range}(\Sigma^\top) = \text{range}(G) \implies \text{range}(\Sigma) = \text{range}(Y_1)
\end{align*}

\paragraph{Calibration.}

From the above, we have established that have $\mathrm{range}(\Sigma) = \mathrm{range}(Y_1)$. For the purposes of \cref{def:strong_calib} we choose $R = Y_1$ and $N = \tilde{Y}_2$ with $\tilde{r} = r$.
We first prove \cref{eq:cond2}. 
The fixed point property says
\begin{align}
    x^\star = G x^\star + f \label{eq:fixed_point}
\end{align}
Left multiplying this with $\tilde{Y}_2^\top$,
\begin{align*}
    \tilde{Y}_2^\top x^\star = \tilde{Y}_2^\top G x^\star + \tilde{Y}_2^\top f
\end{align*}
We thus have
\begin{align*}
    \tilde{Y}_2^\top (x^\star - x) &=   \tilde{Y}_2^\top G x^\star + \tilde{Y}_2^\top f -  \tilde{Y}_2^\top G x_0 - \tilde{Y}_2^\top f \\
    &= \tilde{Y}_2^\top G (x^\star - x_0) \\
    &= 0 
\end{align*}
as $\tilde{Y}_2^\top G = 0$. 
Since this equation is deterministic and $\tilde{Y}_2$ is independent of $x^\star$, this validates \cref{eq:cond2}.

Moving to \cref{eq:cond1}, we have

\begin{align*}
    Y_1^\top \Sigma Y_1 &= Y_1^\top G \Sigma_0 G^\top Y_1 \\
    &= Y_1^\top Y_1 W_1^\top \Sigma_0 W_1 Y_1^\top Y_1
\end{align*}
As $Y_1$ has full rank, $Y_1^\top Y_1$ is symmetric positive definite and $W_1^\top \Sigma_0 W_1$ is positive definite (and thus invertible) as $\Sigma_0$ is positive definite and $W_1$ has full rank \citep[Theorem~4.2.1]{Golub_MatrixComputations}. 
Thus 
\begin{align*}
    (Y_1^\top \Sigma Y_1)^{-1/2} = (W_1^\top \Sigma_0 W_1)^{-1/2} (Y_1^\top Y_1)^{-1}
\end{align*}
This gives us
\begin{align}
    &(Y_1^\top \Sigma Y_1)^{-1/2} Y_1^\top (x^\star - x) \\
    &\qquad= (W_1^\top \Sigma_0 W_1)^{-1/2} (Y_1^\top Y_1)^{-1} Y_1^\top (x^\star - x) \notag\\
    &\qquad= (W_1^\top \Sigma_0 W_1)^{-1/2} (Y_1^\top Y_1)^{-1} Y_1^\top (x^\star - Gx_0 - f) \notag \\
    &\qquad= (W_1^\top \Sigma_0 W_1)^{-1/2} (Y_1^\top Y_1)^{-1} Y_1^\top (x^\star - f) - (W_1^\top \Sigma_0 W_1)^{-1/2} (Y_1 Y_1^\top)^{-1} Y_1^\top Y_1 W_1^\top x_0 \notag\\
    &\qquad= (W_1^\top \Sigma_0 W_1)^{-1/2} (Y_1^\top Y_1)^{-1} Y_1^\top (x^\star - f) - (W_1^\top \Sigma_0 W_1)^{-1/2} W_1^\top x_0 \label{eq:step2}
\end{align}
Left-multiplying \cref{eq:fixed_point} by $Y_1^\top$ gives 
\begin{align*}
    Y_1^\top x^\star &= Y_1^\top G x^\star + Y_1^\top f \\
    Y_1^\top (x^\star - f) &= Y_1^\top Y_1 W_1^\top x^\star
\end{align*}
Substituting this back in \cref{eq:step2} gives us
\begin{align*}
    &(Y_1^\top \Sigma Y_1)^{-1/2} Y_1^\top (x^\star - x) \\
    &\qquad= (W_1^\top \Sigma_0 W_1)^{-1/2} (Y_1^\top Y_1)^{-1} Y_1^\top Y_1 W_1^\top x^\star - (W_1^\top \Sigma_0 W_1)^{-1/2} W_1^\top x_0 \\
    &\qquad= (W_1^\top \Sigma_0 W_1)^{-1/2} W_1^\top (x^\star - x_0)
\end{align*}
We now replace $x^\star$ with $X \sim \normal(x_0, \Sigma_0)$. This gives us $W_1^\top X \sim \normal(W_1^\top x_0, W_1^\top \Sigma_0 W_1)$ and thus $(W_1^\top \Sigma_0 W_1)^{-1/2} (W_1^\top X - W_1^\top x_0) \sim \normal(0, I_{r})$. This proves \cref{eq:cond1} and completes the proof.

\end{proof}

%% file: A03_two_grid_mg.tex
\section{Additional Two-Grid Multigrid Computations} \label{sec:two_grid_derivation}

\cref{alg:multigrid} implements the two-grid cycle of multigrid described in \cref{sec:prob_multigrid_is_hard}.

\begin{algorithm}
    \begin{algorithmic}[1]
        \Procedure{\textsc{TwoGridCycle}}{$x$}
            \State $x^1 \gets \textsc{PreSmooth}_1(x; A^1, b)$
            \State $r^1 \gets b^1 - A^1 x^1$
            \State $r^2 \gets \textsc{Restrict}_{1 \to 2}(r^1)$ \Comment{Restrict residual to coarse grid}
            \State $e^2 \gets \textsc{Solve}(A^2, r^2)$ \Comment{Solve on coarse grid}
            \State $e^1 \gets \textsc{Interpolate}_{2 \to 1}(e^2)$ \Comment{Interpolate correction to fine grid}
            \State $x^1 \gets x^1 + e^1$
            \State $x^1 \gets \textsc{PostSmooth}_1(x^1; A^1, b)$
            \State \Return $x^1$
        \EndProcedure
    \end{algorithmic}
    \caption{Two-grid Cycle for Multigrid}
    \label{alg:multigrid}
\end{algorithm}

\subsection{Probabilistic Two-Grid Multigrid}

In this section we derive the pushforward distribution for a PIM based on two-grid multigrid.
This both emphasises the difficulty in deriving iterations in general, motivating affine tracing, and also is used as a benchmark in \cref{sec: experiments} to demonstrate that tracing-based implementations can be \emph{faster} than hand implementations.
We adopt the same notation as in \cref{sec:prob_multigrid_is_hard} for the multigrid cycle.
To simplify the exposition we abuse notation for affine maps of probability measures; if $\mu$ is a Gaussian measure on $\reals^d$, $B \in \reals^{m \times d}$ and $v \in \reals^d$ then we take $B \mu \equiv B_\sharp \mu$ and $\mu + v \equiv (x \mapsto x + v)_\sharp \mu$.

Given $\mu^1 \sim \mathcal{N}(x^1, \Sigma^1)$, the probabilistic 2-cycle multigrid update can be computed as follows
\begin{itemize}
    \item Presmoothing: $\tilde{\mu}^1 = G{\mu}^1 + f $.
    $$\tilde{\mu}^1 \sim \mathcal{N}(Gx^1 + f, G \Sigma^1 G^\top)$$
    \item Residual: $\rho^1 = b - A \tilde{\mu}^1$.
    $$\rho^1 \sim \mathcal{N}(b - A (G x^1 + f), A G \Sigma^1 (A G)^\top)$$
    \item Course grid restriction: $\rho^2 = R \rho^1$.
    $$\rho^2 \sim \mathcal{N}(Rb - R A (Gx^1 + f), R A G \Sigma^1 (R A G)^\top)$$
    \item Coarse grid solve: $\epsilon^2 = (A^2)^{-1}\rho^2$.
    $$\epsilon^2 \sim \mathcal{N}((A^2)^{-1} Rb - (A^2)^{-1} R A (G x^1 + f), (A^2)^{-1} R A G \Sigma^1 ((A^2)^{-1} R A G)^\top) $$
    \item Interpolate: $\epsilon^1 = P \epsilon^2$
    $$\epsilon^1 \sim \mathcal{N}(P (A^2)^{-1} Rb - P (A^2)^{-1} R A Gx^1, P (A^2)^{-1} R A G \Sigma^1 (P (A^2)^{-1} R A G)^\top)$$
    \item Coarse-grid correction: $\mu^{1, \textsc{cgc}} = \tilde{\mu}^1 + \epsilon^1$.
    \begin{align*}
\mu^{1, \textsc{cgc}} &\sim \mathcal{N}( G x^1 + f + P (A^2)^{-1} Rb - P (A^2)^{-1} R A Gx^1, \Sigma^{1,\textsc{cgc}}) \\
    \Sigma^{1, \textsc{cgc}} &= G \Sigma^1 G^\top + P (A^2)^{-1} R A G \Sigma^1 (P (A^2)^{-1} R A G)^\top - \textup{Cov}(\epsilon_1, \tilde{\mu}_1) - \textup{Cov}(\tilde{\mu}^1, \epsilon^1) \\
    \textup{Cov}(\epsilon_1, \tilde{\mu}_1) &= P (A^2)^{-1} R A G \Sigma^1 G^\top
    \end{align*}
    \item Post-smoothing: $\mu^1_2 = G{\mu^{1, \textsc{cgc}}} + f $
    \begin{align*}
        \mu^1_2 &\sim  \mathcal{N}(x^1_2, \Sigma^{1}_2) \\
        x^1_2 &= G(G x^1 + f + P (A^2)^{-1} Rb - P (A^2)^{-1} R A Gx^1) + f \\
    \Sigma^{1}_2 &= G \left[ G \Sigma^1 (G)^\top + P (A^2)^{-1} R A G \Sigma^1 (P (A^2)^{-1} R A G)^\top - P (A^2)^{-1} R A G \Sigma^1 (G)^\top \right. \\
        &\qquad \left. - (P (A^2)^{-1} R A G \Sigma^1 (G)^\top)^\top  \right]G^\top
    \end{align*}
\end{itemize}
$\mu^1_2 \sim  \mathcal{N}(x^1_2, \Sigma^{1}_2)$ gives the distribution of the update. This is implemented in \cref{alg: pmg}.

\begin{algorithm}
    \begin{algorithmic}[1]
		\Procedure{\textsc{ProbTwoGridCycle}}{$x, \Sigma$}
        \State $x^1, \Sigma^1 = u_0, \Sigma_0$
        \State $x^1, \Sigma^1 = G x^1, G \Sigma^1 G^\top $
        \Comment{Presmoothing}
        \State $\Sigma^1_{r,l} =  A \Sigma$
        \State $r^1, \Sigma^1_r = b - A x^1, \Sigma_l A^\top$
        \Comment{Computing the residual}
        \State $\Sigma^{2}_{r,l} = R \Sigma^1_{r,l}$
        \State $r^2, \Sigma^{2}_r = R r, R \Sigma^1_r R^\top$
        \Comment{Restrict residual to coarse grid}
        \State $\Sigma^{2}_{e,l} = (A^2)^{-1}\Sigma^{2}_{r,l}$
        \State $e^2, \Sigma^{2}_2  = (A^2)^{-1} r^2, (A^2)^{-1} \Sigma^{2}_r ((A^2)^{-1})^\top $
        \Comment{Solve on coarse grid}
        \State $\Sigma^{1}_{e,l} = P \Sigma^{2}_{e,l} $
        \State $e^1, \Sigma^{1}_e = P e^2, P \Sigma^{2}_e P^\top$
        \Comment{Interpolate correction to fine grid}
        \State $x^1, \Sigma^1  = u^1 + e^1, \Sigma^1 + \Sigma^1_e - \Sigma^1_{e,l} - ( \Sigma^1_{e,l})^\top $  
        \Comment{Update iterate}
        \State $x^1, \Sigma^1  = G x^1, G \Sigma^1 G^\top $
        \Comment{Postsmoothing}
		\State \Return $x^1, \Sigma^1$
	\EndProcedure
    \end{algorithmic}
    \caption{Probabilistic Two-grid Multigrid} \label{alg: pmg}
\end{algorithm}

%% file: A02_cagps.tex
\section{Computation-Aware Gaussian Processes} \label{sec: cagp}

An emerging use case of \acp{pls} is for approximation of \acp{GP} \citep{rasmussen_gaussian_2006} through \emph{computation-aware Gaussian processes} \citep{wenger_posterior_2022,wenger_computation-aware_2024,hegde_calibrated_2025,sinaga2025robust}.
We now briefly introduce these methods.

Suppose $f \sim \nu_0 = \mathcal{GP} (m_0, k)$, with a positive definite kernel $k$.
We consider inference under the observations $y = f(X) + \epsilon$, where $X$ is the set of training points and $\epsilon \sim \normal(0, \Gamma)$.
The conditional distribution is given by
\begin{align*}
    \nu^\star &= \mathcal{GP}(\bar{m}, \bar{k})\\
    \bar{m}(x) &= m_0(X') + k(X', X) A^{-1} b \\
    \bar{k}(x, x') &= k(x, x') - k(x, X) A^{-1} k(X, x')
\end{align*}
where $A \defeq k(X, X) + \Gamma$ and $b = y - m_0(X)$.
A CAGP is obtained by applying a \ac{pls} to solve the linear system $Ax = b$, required for computing the posterior mean, and marginalising the uncertainty from the \ac{pls}.
When the \ac{pls} has a posterior of the form $\mu = \normal(\hat{x}, A^{-1} - D)$ (i.e.\ under the inverse prior, as described in \cref{sec:inverse_prior}) a cancellation of the terms involving $A^{-1}$ in the covariance kernel occurs, resulting in the CAGP posterior:
\begin{align*}
    \tilde{\nu} &= \mathcal{GP}(\tilde{m}, \tilde{k})\\
    \tilde{m}(x) &= m_0(x') + k(x, X) \hat{x} \\
    \tilde{k}(x,x') &= k(x, x') - k(x, X) D k(X, x').
\end{align*}

Notably, only the downdate $D$ is required for computation, motivating the interest in this term in \cref{sec:inverse_prior}.
This posterior covariance kernel is wider than the \ac{GP} posterior covariance to account for the additional uncertainty incurred due to reduced computation due stopping the \ac{pls} before convergence.

To implement CAGPs efficiently, it is essential that the term $k(x, X) D k(X, x')$ can be computed efficiently.
This is straightforward when $D$ is low-rank, as is often the case for Bayesian methods.
For non-Bayesian methods, however, $D$ is often full- or near-full rank.
To perform efficient computation one therefore assumes that only low-dimensional summary statistics of the GP are required, e.g.\ its evaluations on a small test set $X'$, computation of which is discussed in \cref{sec:inverse_prior}.

Several different choices of \ac{pls} have been considered in \acp{cagp}.
It has been shown that BayesCG \citep{cockayne_bayesian_2019} provides fast convergence of the \ac{cagp} posterior mean to the truth (as seen in \citet{wenger_posterior_2022}), but the poor calibration properties of BayesCG lead to poor calibration of the \ac{cagp}.
\citet{hegde_calibrated_2025} show that a calibrated \ac{pls} results in calibrated \ac{cagp} (in the same strong calibration sense as \cref{def:strong_calib}).
This was demonstrated using a simple PIM based on Gauss-Seidel, but for that approach mean convergence was unsatisfactory.
Affine tracing enables more straightforward exploration of the set of affine PIMs to identify methods which maintain the calibration guarantees provided in \citet{hegde_calibrated_2025}, while also providing fast convergence.

\cref{alg: CAGP-SIS} describes how to implement a \ac{cagp} using the output of affine tracing, assuming a \emph{stationary} PIM is used.
For that implementation \textsc{at} is assumed to be the output of the affine tracer, after post-processing with \texttt{egglog}, tracing the (stationary) map $P(x) = Gx + f$, where $G = I - M A$ and $S(M) = M + M^\top - M A M^\top$.
We assume that this provides access to the following functions:
\begin{align*}
    \textsc{at.p}(x) &= P(x) \\
    \textsc{at.g}(V) &= V G \\
    \textsc{at.s}(V) &= V S(M).
\end{align*}

\begin{algorithm}[t!]
\caption{Computation Aware GP} \label{alg: CAGP-SIS}
\begin{algorithmic}[1]
    \Function{\textsc{cagp-at}}{$X$, $y$, $X'$, $\Gamma$, \textsc{at}}
    \State $b \gets y - m_0(X)$
    \State $V \gets k(X', X)$
    \State $A(v) \gets v \mapsto k(X, X) v + \Gamma v$
    \State $\tilde{v}, \tilde{D} \gets \textsc{capg-pls-at}(A, b, V, \textsc{at})$
    \State $\tilde{m} \gets m_0(X') + \tilde{v}$
    \State $\tilde{k} \gets k(X', X') - \tilde{D}$
    \State \Return $\tilde{m}, \tilde{k}$
    \EndFunction
    \end{algorithmic}
\begin{algorithmic}[1]
\Function{\textsc{capg-pls-at}}{$A$, $b$, $V$, \textsc{at}}
\State $\tilde{v}_0 \gets 0$
\State $\tilde{D}_0 \gets 0$
\State $Z_1 \gets V$
\For{$i = 1$ to $m$}
    \State $v_i \gets \textsc{at.p}(v_{i-1})$
    \State $\tilde{D}_{i} \gets \tilde{D}_{i-1} + \textsc{at.s}(Z_{i}) \cdot{Z_{i}}$
    \State $Z_{i+1} \gets \textsc{at.g}(Z_i)$
\EndFor
\State \Return $V v_m$, $\tilde{D}_m$
\EndFunction
\end{algorithmic}
\end{algorithm}

\section{Additional Experiments}

\subsection{Additional Results for the Synthetic Problem}

The multigrid components and their configuration highly influence the performance of the multigrid solver. To explore this, we compare different configurations of multigrid components. This study is conducted on a finest training grid of size $70 \times 70$, with the rest of the configuration identical to \cref{sec: synthetic}. \cref{fig: synthetic:rmse_nll_ab} shows the RMSE and NLL of \texttt{CAGP-MG-2} for these methods as a function of the number of iterations, starting from the second iteration. We see that increasing the number of presmoother iterations makes the convergence slightly faster, and the NLL decreases faster. However, as the cost of applying the presmoother dominates the cost of a multigrid iteration, the improved performance would be offset by higher wall-time, which is not apparent from the iteration-based plot. We also see that increasing the number of neighbouring points considered for interpolation and restriction does not seem to have a visible advantage.

Using a coarser second-level grid led to a faster increase in convergence initially. However, the RMSE does not decrease uniformly for these cases. This indicates that the components of multigrid need to be carefully chosen to achieve the best performance from the multigrid method, as well as to preserve convergence guarantees of the solver. Affine tracing enables us to easily fine-tune these components, as long as they are affine, and an exhaustive study of use of multigrid for \acp{cagp} is left for future work.

\begin{figure}
\includegraphics[width=\textwidth]{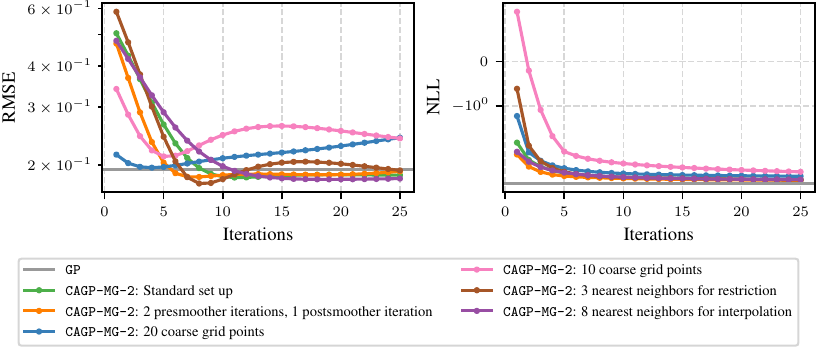}
\caption{Ablation study to test the influence of different multigrid configurations for \texttt{CAGP-MG-2}} \label{fig: synthetic:rmse_nll_ab}
\end{figure}

\subsection{Additional Results for the ERA5 Regression}

Interestingly, \texttt{CAGP-MG-2} in \cref{fig: era5:rmse_nll} appears to stagnate after first few iterations. \cref{fig: era5:rmse_nll_mg_zoom}, the RMSE and NLL plot starting from the second iteration of \texttt{CAGP-MG}, show that RMSE monotonically decreases initially which appears to lead to stagnation. NLL has an initial dip but later follows the same pattern as RMSE. As we also observe in \cref{fig: synthetic:rmse_nll_ab}, the choice of multigrid components and their configuration highly influence the performance of the multigrid solver, and we believe fine-tuning the multigrid components can lead to improved performance of CAGP-MG-2.

\begin{figure}[t!]
\includegraphics[width=\textwidth]{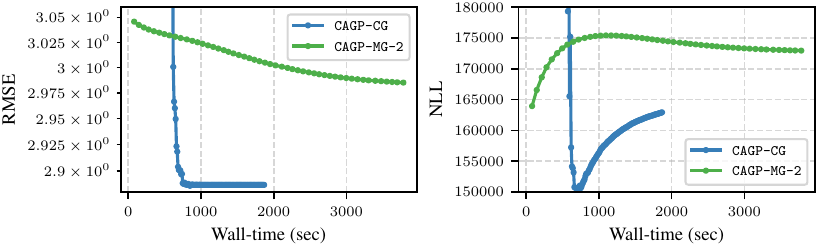}
\caption[RMSE and NLL for the ERA5 regression problem]{RMSE and NLL for the ERA5 regression problem in \cref{sec: era5}, starting from the second iteration of \texttt{CAGP-MG-2}.} \label{fig: era5:rmse_nll_mg_zoom}
\end{figure}

\subsection{Testing Calibration}

\begin{figure}[t!]
    \includegraphics[width=\textwidth]{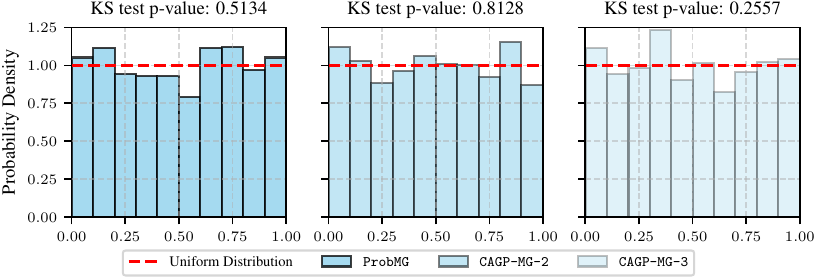}
    \caption{Simulation-based calibration results for the synthetic problem after $m = 3$ iterations}
    \label{fig: synthetic:sbc}
\end{figure}

Here we test the calibration guarantees of the affine probabilistic linear solver (as given by \cref{thm:affine_calibrated}) and the resulting \ac{cagp} posterior (as established in \citet[Theorem 4]{hegde_calibrated_2025}). We use the simulation-based calibration (SBC) test \citep{Talts2018} for this purpose. The SBC test samples the true function from the prior multiple times and checks the distribution of a particular test statistic based on the probability integral transform. Calibration is indicated by the uniformity of this statistic.
\cref{fig: synthetic:sbc} shows the results of using SBC to test the calibration of an affine traced probabilistic multigrid solver (\texttt{ProbMG} for the problem in \cref{sec: time_affine_tracer}, with $d = 1225$) \texttt{CAGP-MG-2} (for the problem in \cref{sec: synthetic}, with the finest training grid of size $35 \times 35$ and coarsest training grid of size $10 \times 10$ and test grid of size $5 \times 5$), and \texttt{CAGP-MG-3} (for the problem in \cref{sec: synthetic}, with the finest training grid of size $35 \times 35$, second grid of size $10 \times 10$ and coarsest grid of size $5 \times 5$ and test grid of size $5 \times 5$). We also report the p-values for the Kolmogorov-Smirnov (KS) test for uniformity. \cref{fig: synthetic:sbc} shows that the distributions of test statistics for \texttt{ProbMG}, \texttt{CAGP-MG-2} and \texttt{CAGP-MG-3} look uniform, and thus texttt{ProbMG}, \texttt{CAGP-MG-2} and \texttt{CAGP-MG-3} are calibrated, supporting the theoretical guarantees. The p-values of a KS tests show that uniformity can't be rejected.

\subsection{Egglog Compilation Times}

\cref{table: times} below shows an example of compilation times for some iterative solvers. We note that these are independent of the size of the training dataset, i.e., the size of the linear system.
For 2- and 3- grid multigrid, 1 pre- and post-smoother iteration were used.
We note that the number of pre- and post-smoother iterations affects the size of the computational graph and therefore the compilation time.
However, if the smoothers are stationary it is likely that further optimisations of the pipeline could mitigate this.

\begin{table} [h!]
    \centering
    \caption{Table of egglog compilation times.
    \vspace{5pt}}
    \begin{tabular}{c|c} \label{table: times}
        Method & Compilation Time \\ \hline
        GS  & 479 ms \\
        2-grid Multigrid & 24.1 s \\
        3-grid Multigrid & 3 min 5 s
    \end{tabular}
\end{table}